\documentclass[10pt,journal,compsoc]{IEEEtran}
\usepackage[english]{babel}
\usepackage[utf8]{inputenc}
\usepackage[]{algorithm2e}

\usepackage{amsthm}

\ifCLASSINFOpdf
   \usepackage[pdftex]{graphicx}
\else

\fi

\usepackage{amsmath}
\usepackage{dsfont} 
\usepackage{amssymb} 

\newtheorem{theorem}{Theorem}[section]
\newtheorem{definition}[theorem]{Definition}
\newtheorem{lemma}[theorem]{Lemma} 
\newtheorem{conjecture}[theorem]{Conjecture}
\newtheorem{corollary}[theorem]{Corollary}

\newcommand{\N}{\mathbb{N}} 
\newcommand{\R}{\mathbb{R}} 

\usepackage{tikz}
\newcommand\circlearound[1]{%
  \tikz[baseline]\node[draw,shape=circle,anchor=base] {#1} ;}

\begin{document}
\title{An improvement of the convergence proof of the ADAM-Optimizer}
\author{Sebastian~Bock, Josef~Goppold, Martin~Weiß \\
\small \textit{Ostbayerische Technische Hochschule (OTH) Regensburg, Germany}\\
\textit{\{sebastian2.bock, martin.weiss\}@oth-regensburg.de}, goppold@mediamarktsaturn.com}

\markboth{Conference paper at OTH Clusterkonferenz 2018, 13.04.2018}%
{Conference paper at OTH Clusterkonferenz 2018, 13.04.2018}
\IEEEtitleabstractindextext{%
\begin{abstract}
A common way to train neural networks is the Backpropagation. This algorithm includes a gradient descent method, which needs an adaptive step size. In the area of neural networks, the ADAM-Optimizer is one of the most popular adaptive step size methods. It was invented in \cite{Kingma.2015} by Kingma and Ba. The $5865$ citations in only three years shows additionally the importance of the given paper. We discovered that the given convergence proof of the optimizer contains some mistakes, so that the proof will be wrong. In this paper we give an improvement to the convergence proof of the ADAM-Optimizer.
\end{abstract}

}

\maketitle

\IEEEdisplaynontitleabstractindextext

\IEEEpeerreviewmaketitle

\ifCLASSOPTIONcompsoc
\IEEEraisesectionheading{\section{Introduction}\label{sec:introduction}}
\else
\section{Introduction}
\label{sec:introduction}
\fi

\IEEEPARstart{N}{owadays} machine learning and artificial intelligence are very popular techniques but there is still a lot of research to do. To make methods like neural networks usable, we have to use learning algorithms, like the Backpropagation. Backpropagation is a kind of gradient descent method. In order to improve the convergence of such methods, it is a common way to introduce an adaptive step size. Adaptive step size is a numerical process to solve continuous problems with a discretization in single steps. Computation of the required step size, is still a big problem and there are many possible ways to define them. In this paper we discuss the ADAM-Optimizer from Kingma and Ba \cite{Kingma.2015}. The ADAM-Optimizer is one of the most popular gradient descent optimization
algorithms. It is implemented in common neural network frameworks, like TensorFlow, Caffe or CNTK. Kingma and Ba show experimentally, that the ADAM-Optimizer is faster than any other Optimizer (see figure \ref{fig:training_cost}).
\begin{figure}[h]
	\centering
	\includegraphics[width=7cm]{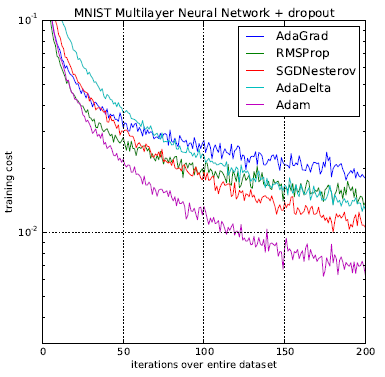}
	\caption{Comparison of different optimizer by training of multilayer neural networks on MNIST images. (Image from \cite{Kingma.2015})}
	\label{fig:training_cost}
\end{figure}
Sebastian Ruder says in \cite{Ruder.2017} "Insofar, Adam might be the best overall choice". All these points express the importance of this optimizer for neural networks. Independently of each other Josef Goppold and Sebastian Bock found out in their Master theses \cite{Goppold.2017} and \cite{Bock.2017} , that there are some mistakes in the convergence proof from Kingma and Ba. Even though we can not solve the proof completely, we achieve an improvement in some parts and can formulate a single conjecture, which would complete the proof.
\section{Neural Networks}
In neural networks we have a group of neurons and everyone of them has a weight $w$, which will be stored in the weight vector $w \in \R^n$. In the learning phase we modify this vector to obtain a network with the required intelligence. In order to evaluate the neural network with the current weight vector, we define an error function $e(w)$. This error function shall compare the label of the input with the output of the network. A popular method to minimize $e(w)$ is the Backpropagation, which uses the gradient descent method. At this point we can use the ADAM-Optimizer.
\section{Method of moments - ADAM}
\subsection{Method of moments}
The method of moments is based on an adaptive step size. At first we define our weight change rule.
\begin{definition} \textbf{(Weight change rule)}\\
\label{def:weightchangerule}
Let $w \in \R^n$ be the weight vector of our neural network, $e(w)$ the error function and $\eta \in \R^+$ the step size. Moreover let $t \in \N$ be the time stamp of the current training step. Then is $w(t)$ the weight vector in the training step $t$.
\begin{align*}
w(t+1) := w(t) + \Delta w(t) \qquad \text{with} \quad \Delta w(t):=-\frac{\eta}{2} \nabla_w e(w(t))
\end{align*}
\end{definition}
With a rule like in definition \ref{def:weightchangerule} we can improve our weights to minimize the error of our neural network. The shape of $\nabla w(t)$ depends on the chosen method. In our case the method of moments. \\ \\
The method of moments adds to the gradient descent step a fraction of the weight changes from the time stamp before. Mathematically it looks like:
\begin{definition}\textbf{(Method of moments)}\\
Let $\alpha \in \R^+$ be the decay rate of the old weight change. Furthermore let all parameters be defined as in definition \ref{def:weightchangerule}. Then the weight change will be defined as follows:
\begin{align*}
\Delta w \left(t \right) := -\frac{\eta}{2} \nabla_w e \left(w \left(t \right) \right) + \alpha \Delta w \left(t-1 \right)
\end{align*}
\end{definition}
In order to attain convergence of the method of moments, the restriction $\alpha \in ]0,1[$ should be applied.
\subsection{ADAM-Optimizer}
The adaptive moment estimization (ADAM) was invented by Kingma and Ba \cite{Kingma.2015} and is nowadays one of the most popular step size methods in the area of neural networks. The algorithm is defined as follows.
\begin{algorithm}
\label{algo:ADAM}
\KwData{$\eta_t := \frac{\eta}{\sqrt{t}}$ as step size, $\beta_1, \beta_2 \in (0,1)$ as decay rates for the moment estimates, $\beta_{1,t} := \beta_1 \lambda^{t-1}$ with $\lambda \in (0,1)$, $\epsilon > 0$, $e(w(t))$ as a convex differentiable error function and $w(0)$ as the initial weight vector.}
Set $m_0 =0$ as initial $1^{st}$ moment vector\\
Set $v_0 =0$ as initial $2^{nd}$ moment vector\\
Set $t =0$ as initial time stamp\\
\While{$w(t)$ not converged}{
$t = t+1$\\
$g_t = \nabla_w e(w(t-1))$\\
$m_t = \beta_{1,t} m_{t-1} + (1-\beta_{1,t})g_t$\\
$v_t = \beta_2 v_{t-1} + (1-\beta_2)g_t^2$\\
$\hat{m}_t = \frac{m_t}{(1-\beta_1^t)}$\\
$\hat{v}_t = \frac{v_t}{(1-\beta_2^t)}$\\
$w(t) = w(t-1) - \eta_t \frac{\hat{m}_t}{(\sqrt{\hat{v}_t} + \epsilon)}$}
\KwRet{$w(t)$}
\caption{ADAM-Optimizer}
\end{algorithm}
In \cite{Kingma.2015} they show experimentally, that the ADAM-Optimizer converges much faster for multi-layer neural networks or convolutional neural networks, than any other optimizer. Unfortunately there are some mistakes in the convergence proof of the paper \cite{Kingma.2015}, so that the proof fails to be correct. In this paper we introduce an improvement of the convergence proof of the ADAM-Optimizer.
\section{Convergence proof}
First of all, recall the following lemma which will give us an odd entrance to convex functions.
\begin{lemma}
\label{lemma:convex_function}
Let $D \subset \R^n$ be a convex set and $f \in C^1(\R^n, \R)$. Then $f$ is a convex function on $D$ if and only if the following condition holds:
\begin{align*}
f(y) \geq f(x) + \nabla f(x)^T (y-x)
\end{align*}
$\forall x,y \in D$ with $x \neq y$.
\end{lemma}
A proof of this lemma may be found in \cite{Forster.2016} site 37. In the following $e$ denotes a convex and differentiable function and $g_t := \nabla e_t(\overrightarrow{w}(t))$ is the gradient of $e$ at the times tamp $t$. Additional let $g_{t,i}$ be the $i$th element of the gradient and $g_{1:t,i} := (g_{1,i}, g_{2,i}, \cdots, g_{t,i})^T \in \R^t$. The described lemma $10.4$ in \cite{Kingma.2015} could unfortunately not be proven and we will refer to it as a conjecture.
\begin{conjecture}
\label{conjecture:Vermutung}
Let $\gamma := \frac{\beta_1^2}{\sqrt{\beta_2}}$ with $\beta_1, \beta_2 \in (0,1)$ and $\gamma < 1$. Moreover let $g_t$ be bounded with $||g_t||_2 \leq G$ and $||g_t||_\infty \leq G_\infty$. Then,
\begin{align*}
\sum\limits_{t=1}^T \frac{\hat{m}_{t,i}^2}{\sqrt{t \hat{v}_{t,i}}} \leq \frac{2}{(1 - \gamma)} \frac{1}{\sqrt{1-\beta_2}} ||g_{1:T,i}||_2
\end{align*}
\end{conjecture}
In the next step we will define an error sum, which calculates the difference between the minimum and the current value of $e\left(w\left(t\right)\right)$.
\begin{definition}\textbf{(Error sum)} \\
\label{def:Error_sum}
Let $\overrightarrow{w}^* := \arg \min\limits_{\overrightarrow{w} \in \chi} \sum\limits_{t=1}^{T} e_t(\overrightarrow{w})$ with $\chi$ as the set of $\overrightarrow{w}$, which will arise in the ADAM-Method. The error sum is then defined as:
\begin{align*}
R(T) := \sum\limits_{t=1}^T \left( e_t(\overrightarrow{w}_t)- e_t(\overrightarrow{w}^*) \right)
\end{align*}
\end{definition}
If we are able to show the convergence of $R(T)$ with respect to $T$, the convergence proof is done. We will do this with the following theorem.
\begin{theorem}
\label{theorem:Main_Theorem}
Let $g_t$ be bounded with $||g_t||_2 \leq G$ and $||g_t||_\infty \leq G_\infty$ for all $t \in \{1, \cdots, T \}$. Furthermore, suppose that the difference between $\overrightarrow{w}_t$ is bounded by $||\overrightarrow{w}_n - \overrightarrow{w}_m||_2 \leq D$ and $||\overrightarrow{w}_n - \overrightarrow{w}_m||_\infty \leq D_\infty$ with $n,m \in \{1, \cdots, T \}$. Furthermore let $\beta_1, \beta_2 \in (0,1)$, $\gamma := \frac{\beta_1^2}{\sqrt{\beta_2}} < 1$, $\eta_t := \frac{\eta}{\sqrt{t}}$ and $\beta_{1,t} := \beta_1 \lambda^{t-1}$ with $\lambda \in (0,1)$. Then the ADAM-Optimizer can be estimated as follows:
\begin{align*}
R(T) \leq& \frac{D_\infty^2}{2 \eta (1-\beta_1)} \sum\limits_{i=1}^d \sqrt{T \hat{v}_{T,i}} + \frac{d D^2_\infty G_\infty}{2 \eta (1-\beta_1)(1-\lambda)^2}\\ &+ \frac{\eta (\beta_1 +1)}{(1- \beta_1)\sqrt{1- \beta_2} (1-\gamma)} \sum\limits_{i=1}^d ||g_{1:T,i}||_2
\end{align*}
\end{theorem}
\begin{proof}
With lemma \ref{lemma:convex_function} we can write for a convex differentiable function $e(w)$:
\begin{align*}
e_t \left(\overrightarrow{w}^* \right) &\geq e_t \left(\overrightarrow{w}_t\right) + g_t^T \left(\overrightarrow{w}^* - \overrightarrow{w}_t \right)\\
\Leftrightarrow e_t (\overrightarrow{w}_t) - e_t (\overrightarrow{w}^*) &\leq g_t^T \left(\overrightarrow{w}_t - \overrightarrow{w}^* \right)
\end{align*}
With the update rule from the ADAM-Optimizer:
\begin{align*}
\overrightarrow{w}_{t+1} &= \overrightarrow{w}_t - \eta_t \frac{\hat{m}_t}{\sqrt{\hat{v}_t}}\\
&= \overrightarrow{w}_t - \frac{\eta_t}{1 - \beta_1^t} \left(\frac{\beta_{1,t}}{\sqrt{\hat{v}_t}}m_{t-1} + \frac{(1- \beta_{1,t})}{\sqrt{\hat{v}_t}} g_t \right)
\end{align*}
Now we consider the $i$th component of $\overrightarrow{w}_t \in \R^d$.
\begin{align*}
\overrightarrow{w}_{t+1,i} - \overrightarrow{w}^*_{,i} =& \overrightarrow{w}_{t,i}- \overrightarrow{w}_{,i}^* - \eta_t \frac{\hat{m}_{t,i}}{\sqrt{\hat{v}_{t,i}}}\\
(\overrightarrow{w}_{t+1,i} - \overrightarrow{w}^*_{,i})^2 =& (\overrightarrow{w}_{t,i} - \overrightarrow{w}_{,i}^*)^2 - \frac{2 \eta_t \hat{m}_{t,i}}{\sqrt{\hat{v}_{t,i}}} + \eta_t^2 \left( \frac{\hat{m}_{t,i}}{\sqrt{\hat{v}_{t,i}}} \right)^2\\
g_{t,i} \left( \overrightarrow{w}_{t,i} - \overrightarrow{w}_{,i}^* \right) =& \frac{(1-\beta_1^t) \sqrt{\hat{v}_{t,i}}}{2 \eta_t (1-\beta_{1,t})} \left( \left( \overrightarrow{w}_{t,i} - \overrightarrow{w}_{,i}^* \right)^2 - \left( \overrightarrow{w}_{t+1} - \overrightarrow{w}_{,i}^* \right)^2 \right) \\
&- \underbrace{\frac{\beta_{1,t}}{(1-\beta_{1,t})} m_{t-1,i} \left( \overrightarrow{w}_{t,i} - \overrightarrow{w}_{,i}^* \right)}_{(*)} \\
&+ \frac{\eta_t \left( 1 - \beta_1^t \right) \sqrt{\hat{v}_{t,i}} }{2 \left( 1 - \beta_{1,t} \right)} \left( \frac{\hat{m}_{t,i}}{\sqrt{\hat{v}_{t,i}}} \right)^2
\end{align*}
In $(*)$ we multiply with $1 = \frac{\hat{v}_{t-1}^{\frac{1}{4}} \sqrt{\eta_{t-1}}}{\hat{v}_{t-1}^{\frac{1}{4}} \sqrt{\eta_{t-1}}}$ and use the binomial equation to simplify:
\begin{align*}
&\frac{\beta_{1,t}}{1-\beta_{1,t}} \left( \overrightarrow{w}^*_{,i} - \overrightarrow{w}_{t,i} \right) \frac{\hat{v}_{t-1}^{\frac{1}{4}} \sqrt{\eta_{t-1}}}{\hat{v}_{t-1}^{\frac{1}{4}} \sqrt{\eta_{t-1}}} =\\
 &= \frac{\beta_{1,t}}{1 - \beta_{1,t}} \left( \frac{\hat{v}_{t-1,i}^{\frac{1}{4}}}{\sqrt{\eta_{t-1}}} \left( \overrightarrow{w}_{,i}^* - \overrightarrow{w}_{t,i} \right) \sqrt{\eta_{t-1}} \frac{m_{t-1,i}}{\hat{v}_{t-1,i}^{\frac{1}{4}}} \right) \\
& \leq \underbrace{\frac{\beta_{1,t}}{1 - \beta_{1,t}}}_{\leq \frac{\beta_1}{1 - \beta_1}} \left( \frac{\sqrt{\hat{v}_{t-1},i} \left( \overrightarrow{w}_{,i}^*  - \overrightarrow{w}_{t,i} \right)^2}{2 \eta_{t-1}} + \frac{\eta_{t-1} m_{t-1,i}}{2 \sqrt{\hat{v}_{t-1,i}}} \right)
\end{align*}
If we put all these together we reach the following inequality. We separate it in five terms. Each of them will be handled on their own.
\begin{align*}
\underbrace{g_{t,i} \left( \overrightarrow{w}_{t,i} - \overrightarrow{w}_{,i}^* \right) }_{\circlearound{1}} &\leq \underbrace{\frac{\left( \left( \overrightarrow{w}_{t,i} - \overrightarrow{w}_{,i}^* \right)^2 - \left( \overrightarrow{w}_{t+1,i} - \overrightarrow{w}_{,i}^* \right)^2 \right) \sqrt{\hat{v}_{t,i}}}{2 \eta_t \left( 1 - \beta_1 \right) }}_{\circlearound{2}} \\
&+ \underbrace{\frac{\beta_{1,t}}{2 \eta_{t-1} \left( 1 - \beta_{1,t} \right)} \left( \overrightarrow{w}_{,i}^* - \overrightarrow{w}_{t,i} \right)^2 \sqrt{\hat{v}_{t-1,i}}}_{\circlearound{3}}\\
&+ \underbrace{\frac{\beta_1 \eta_{t-1} m^2_{t-1,i}}{2 \left( 1 - \beta_1 \right) \sqrt{\hat{v}_{t-1,i}}}}_{\circlearound{4}} + \underbrace{\frac{\eta_t \hat{m}^2_{t,i}}{2 \left( 1 - \beta_{1} \right) \sqrt{\hat{v}_{t,i}}}}_{\circlearound{5}}
\end{align*}
To get the link to the error sum, we sum over the elements of the gradient $i \in 1, \cdots, d$ and the time stamps $t \in 1, \cdots, T$. Then term $\circlearound{1}$ looks like:
\begin{align*}
\sum\limits_{t = 1}^T \sum\limits_{i = 1}^d g_{t,i} \left( \overrightarrow{w}_{t,i} - \overrightarrow{w}_{,i}^* \right) &= \sum\limits_{t = 1}^T g_t^T \left( \overrightarrow{w}_t - \overrightarrow{w}^* \right)\\
&\geq \sum\limits_{t = 1}^T \left( e_t \left( \overrightarrow{w}_t \right) - e_t \left( \overrightarrow{w}^* \right)\right)\\
&= R(T)
\end{align*}
Now we look at term $\circlearound{2}$. 
\begin{align*}
&\sum\limits_{i = 1}^d \sum\limits_{t=1}^T \frac{\left( \left( \overrightarrow{w}_{t,i} - \overrightarrow{w}^*_{,i} \right)^2 - \left( \overrightarrow{w}_{t+1,i} - \overrightarrow{w}_{,i}^* \right)^2 \right) \sqrt{\hat{v}_{t,i}}}{2 \eta_t \left( 1 - \beta_1 \right)} \\
&= \sum\limits_{i=1}^d \frac{1}{2 \eta_1 \left( 1 - \beta_1 \right)} \left( \overrightarrow{w}_{1,i} - \overrightarrow{w}_{,i}^* \right)^2 \sqrt{\hat{v}_{1,i}}\\
&+ \sum\limits_{i=1}^d \sum\limits_{t=2}^T \frac{1}{2 \eta_t \left( 1- \beta_1 \right)} \left( \overrightarrow{w}_{1,i} - \overrightarrow{w}_{,i}^* \right)^2 \sqrt{\hat{v}_{1,i}}\\
&- \underbrace{\sum\limits_{i=1}^d \sum\limits_{t=1}^T \frac{1}{2 \eta_t}\left( \overrightarrow{w}_{t+1,i} - \overrightarrow{w}^*_{,i} \right)^2 \sqrt{\hat{v}_{t,i}}}_{\circlearound{2a}}
\end{align*}
We can rewrite $\circlearound{2a}$:
\begin{align*}
\circlearound{2a} &= \sum\limits_{i=1}^d \sum\limits_{t=1}^T \frac{1}{2 \eta_{t-1} \left( 1-\beta_{1} \right)} \left( \overrightarrow{w}_{t,i} - \overrightarrow{w}^*_{,i} \right)^2 \sqrt{\hat{v}_{t-1,i}}\\
&+ \sum\limits_{i=1}^d \frac{1}{2 \eta_{T} \left( 1-\beta_{1} \right)} \left( \overrightarrow{w}_{T+1,i} - \overrightarrow{w}^*_{,i} \right)^2 \sqrt{\hat{v}_{T,i}}
\end{align*}
After all, the following results for $\circlearound{2}$.
\begin{align*}
\circlearound{2} &= \sum\limits_{i=1}^d \frac{1}{2 \eta_{1} \left(1-\beta_1 \right)} \underbrace{\left(\overrightarrow{w}_{1,i} - \overrightarrow{w}_{,i}^* \right)^2}_{\leq D^2_\infty} \sqrt{\hat{v}_{1,i}}\\
&+ \sum\limits_{i=1}^d \sum\limits_{t=2}^T \frac{1}{2 \left(1-\beta_1 \right)} \underbrace{\left( \overrightarrow{w}_{t,i} - \overrightarrow{w}_{,i}^* \right)^2}_{\leq D^2_\infty} \left( \frac{\sqrt{\hat{v}_{t,i}}}{\eta_t} - \frac{\sqrt{\hat{v}_{t-1,i}}}{\eta_{t-1}}\right)\\
&\underbrace{- \sum\limits_{i=1}^d \frac{1}{2 \eta_T \left(1-\beta_1 \right)} \left( \overrightarrow{w}_{T+1,i} - \overrightarrow{w}_{,i}^* \right)^2 \sqrt{\hat{v}_{T,i}}}_{\leq 0}\\
&\leq \frac{D_\infty^2}{2 \eta (1-\beta_1)} \left( \sum\limits_{i=1}^d \sqrt{\hat{v}_{1,i}} + \sum\limits_{i=1}^d \sum\limits_{t=2}^T \left( \sqrt{t \hat{v}_{t,i}} - \sqrt{(t-1) \hat{v}_{t-1,i}} \right) \right)\\
&= \frac{D^2_\infty}{2 \eta \left( 1-\beta_1 \right)} \sum\limits_{i=1}^d \sqrt{T \hat{v}_{T,i}}
\end{align*}
Now we look at term $\circlearound{3}$.
\begin{align*}
\circlearound{3} &\leq \sum\limits_{i=1}^d \sum\limits_{t=1}^T \frac{\beta_{1,t}}{2 \eta_t \left( 1-\beta_{1,t} \right)} \left( \overrightarrow{w}_{,i}^* - \overrightarrow{w_{t,i}} \right)^2 \sqrt{\hat{v}_{t-1,i}}\\
&= \frac{1}{2 \eta} \sum\limits_{t=1}^T \sum\limits_{i=1}^d \underbrace{ \left( \overrightarrow{w}_{,i}^* - \overrightarrow{w}_{t,i} \right)^2}_{\leq D^2_\infty}  \frac{\beta_{1,t}}{1 - \beta_{1,t}} \sqrt{t \hat{v}_{t-1,i}}\\
&\leq \frac{D_\infty^2}{2 \eta} \sum\limits_{i=1}^d \sum\limits_{t=1}^T \frac{\beta_{1,t}}{\left( 1-\beta_{1,t} \right)} \sqrt{t\hat{v}_{t-1,i}}
\end{align*}
With
\begin{align*}
\sqrt{\hat{v}_{t-1,i}} &= \sqrt{1-\beta_2} \sqrt{\frac{\sum\limits_{j=1}^{t-1} g_{j,i}^2 \beta_2^{t-1-j}}{1-\beta_2^{t-1}}}\\
&\leq \sqrt{1-\beta_2} G_\infty \sqrt{\frac{\sum\limits_{j=1}^{t-1} \beta_2^{t-1-j}}{1-\beta_2^{t-1}}}\\
&\leq \sqrt{1-\beta_2} G_\infty \sqrt{\frac{\sum\limits_{j=1}^{t-1} \beta_2^{j}}{1 -\beta_2^{t-1}}}\\
&\leq \sqrt{1-\beta_2} G_\infty \sqrt{\frac{1-\beta_2^{t-1}}{\left( 1-\beta_2^{t-1} \right) \left( 1-\beta_2 \right)}}\\
&\leq G_\infty
\end{align*}
follows
\begin{align*}
\circlearound{3} \leq \frac{D_\infty^2 G_\infty}{2 \eta} \sum\limits_{i=1}^d \sum\limits_{t=1}^T \frac{\beta_{1,t}}{1-\beta_{1,t}} \sqrt{t}
\end{align*}
For $\sum\limits_{t=1}^T \frac{\beta_{1,t}}{\left( 1- \beta_{1,t} \right)} \sqrt{t}$ we can estimate:
\begin{align*}
\sum\limits_{t=1}^T \frac{\beta_{1,t}}{\left( 1- \beta_{1,t} \right)} \sqrt{t} &\leq \sum\limits_{t=1}^T \frac{\beta_{1} \lambda^{t-1}}{\left( 1- \beta_{1} \right)} \sqrt{t} \\
&\leq \sum\limits_{t=1}^T \frac{\lambda^{t-1}}{\left( 1- \beta_{1} \right)} t\\
&= \frac{1}{1- \beta_{1}} \sum\limits_{t=0}^{T-1} \lambda^t \left( t+1 \right)\\
&= \frac{1}{1- \beta_{1}} \left( \sum\limits_{t=0}^{T-1} \lambda^tt + \sum\limits_{t=0}^{T-1} \lambda^t \right)\\
&= \frac{\left( \frac{\left( T-1\right) \lambda^{T+1} - T\lambda^T + \lambda}{\left( \lambda-1 \right)^2} + \frac{1-\lambda^T}{1-\lambda} \right)}{1- \beta_{1}} \\
&= \frac{\left(1- \underbrace{T\left( \lambda^T - \lambda^{T+1} \right)}_{\geq 0} - \underbrace{\lambda T}_{\geq 0} \right)}{\left( 1-\beta_1 \right) \left( \lambda -1 \right)^2} \\
&\leq \frac{1}{\left( 1-\beta_1 \right) \left( \lambda -1 \right)^2}
\end{align*}
Then $\circlearound{3}$ results in:
\begin{align*}
\circlearound{3} \leq \sum\limits_{i=1}^d \frac{D_\infty^2 G_\infty}{2 \eta \left( 1- \beta_1 \right) \left( 1- \lambda \right)^2} = \frac{d D_\infty^2 G_\infty}{2 \eta \left( 1-\beta_1 \right) \left( 1-\lambda \right)^2}
\end{align*}
For term $\circlearound{4}$ we estimate:
\begin{align*}
\circlearound{4} &= \frac{\beta_1 \eta}{2 \left(1-\beta_1 \right)} \sum\limits_{i=1}^d \sum\limits_{t=1}^T \frac{\hat{m}_{t-1,i}^2}{\sqrt{\left( t-1 \right) \hat{v}_{t-1,i}}}\\
&= \frac{\beta_1 \eta}{2 \left( 1-\beta_1 \right)} \sum\limits_{i=1}^d \sum\limits_{t=1}^T \frac{\hat{m}^2_{t-1,i}}{\sqrt{\left(t-1 \right) \hat{v}_{t-1,i}}} \underbrace{\left( 1-\beta_1^{t-1} \right)^2}_{\leq 1}\\
&\leq \frac{\beta_1 \eta}{2 \left( 1- \beta_1 \right)} \sum\limits_{i=1}^d \frac{2}{\left( 1- \gamma \right) \sqrt{1-\beta_2}} ||g_{1:T,i}||_2\\
&= \frac{\beta_1 \eta}{\left(1-\beta_1 \right) \sqrt{1-\beta_2} \left(1-\gamma \right)} \sum\limits_{i=1}^d ||g_{1:t,i}||_2
\end{align*}
Analogously to $\circlearound{4}$, for $\circlearound{5}$:
\begin{align*}
\sum\limits_{i=1}^d \sum\limits_{t=1}^T \frac{\eta_t}{2 \left(1-\beta_1 \right)} \frac{\hat{m}_{t,i}^2}{\sqrt{\hat{v}_{t,i}}} &= \frac{\eta}{2 \left(1-\beta_1 \right)} \sum\limits_{i=1}^d \sum\limits_{t=1}^T \frac{\hat{m}_{t,i}^2}{\sqrt{t\hat{v}_{t,i}}}\\
&\leq \frac{\eta}{2 \left(1-\beta_1 \right)} \sum\limits_{i=1}^d \frac{2 ||g_{1:T,i}||_2}{\left(1-\gamma\right) \sqrt{1-\beta_2}}\\
&= \frac{\eta \sum\limits_{i=1}^d ||g_{1:T,i}||_2}{\left( 1-\beta_1 \right) \sqrt{1-\beta_2} \left( 1-\gamma \right)}
\end{align*}
Both in $\circlearound{4}$ and in $\circlearound{5}$ we use conjecture \ref{conjecture:Vermutung}. Now we can combine both.
\begin{align*}
\circlearound{4} + \circlearound{5} = \frac{\eta \left(1+\beta_1 \right)}{\left(1-\beta_1 \right) \sqrt{1-\beta_2} \left(1-\gamma \right)} \sum\limits_{i=1}^d ||g_{1:T,i}||_2
\end{align*}
If we combine all terms, we get our assertion and the proof is finished.
\begin{align*}
R(T) \leq& \frac{D^2_\infty}{2 \eta \left( 1-\beta_1 \right)} \sum\limits_{i=1}^d \sqrt{T\hat{v}_{T,i}} + \frac{d D^2_\infty G_\infty}{2 \eta \left( 1-\beta_1 \right) \left(1-\lambda \right)^2}\\
&+ \frac{\eta \left(1+\beta_1 \right)}{\left(1-\beta_1 \right) \sqrt{1-\beta_2} \left(1-\gamma \right)} \sum\limits_{i=1}^d ||g_{1:T,i}||_2
\end{align*}
\end{proof}
Using Theorem \ref{theorem:Main_Theorem} we can prove the following corollary
\begin{corollary}
Let $e_t$ with $t=1, \cdots, T$ be convex with a bounded gradient $||\nabla e_t \left( \overrightarrow{w} \right)||_2 \leq G$, $||\nabla e_t \left( \overrightarrow{w} \right) ||_\infty \leq G_\infty$, $ \forall \overrightarrow{w} \in \R^d$. Furthermore, suppose the difference between $\overrightarrow{w}_t$ is bounded by $|| \overrightarrow{w}_n - \overrightarrow{w}_m||_2 \leq D$, $||\overrightarrow{w}_n - \overrightarrow{w}_m||_\infty \leq D_\infty$, $\forall m,n \in 1,\cdots, T$. Then the following convergence estimation for the ADAM-Method $\forall T \geq 1$ holds:
\begin{align*}
\frac{R(T)}{T} = \mathcal{O} \left( \frac{1}{\sqrt{T}} \right)
\end{align*}
\end{corollary}
\begin{proof}
The same requirements apply as above. Then the inequality from theorem \ref{theorem:Main_Theorem} applies and because of $T > 0$ we can divide by $T$.
\begin{align*}
\frac{R(T)}{T} &\leq \frac{D^2_\infty}{2 \eta \left(1-\beta_1 \right)} \sum\limits_{i=1}^d \frac{\sqrt{\hat{v}_{T,i}}}{\sqrt{T}} + \frac{d D_\infty^2 G_\infty}{T 2 \eta \left(1-\beta_1 \right) \left(1-\lambda \right)^2}\\
&+ \frac{\eta \left(1+\beta_1 \right)}{T \left(1-\beta_1 \right) \sqrt{1-\beta_2} \left(1-\gamma \right)} \sum\limits_{i=1}^d ||g_{1:T,i}||_2
\end{align*}
With
\begin{align*}
\sum\limits_{i=1}^d ||g_{1:T,i}||_2 &= \sum\limits_{i=1}^d \sqrt{g_{1,i}^2 + g_{2,i}^2 + \cdots + g_{T,i}^2}\\
&\leq \sum\limits_{i=1}^d \sqrt{G^2_\infty + G^2_\infty + \cdots + G^2_\infty}\\
&= \sum\limits_{i=1}^d \sqrt{T}G_\infty\\
&= dG_\infty\sqrt{T}
\end{align*}
and
\begin{align*}
\sum\limits_{i=1}^d \sqrt{T\hat{v}_{T,i}} &\leq \sum\limits_{i=1}^d \sqrt{T}G_\infty\\
&\leq dG_\infty \sqrt{T}
\end{align*}
we can estimate:
\begin{align*}
\lim\limits_{T \rightarrow \infty}\frac{R(T)}{T} \leq \lim\limits_{T \rightarrow \infty}  \left(  \frac{1}{\sqrt{T}} + \frac{1}{\sqrt{T}} + \frac{1}{T} \right) = 0
\end{align*}
This proves the convergence speed $\mathcal{O} \left( \frac{1}{\sqrt{T}} \right)$ of the ADAM-Method.
\end{proof}
\section{Conclusion and outlook}
Machine learning and particularly neural networks are advancing fast. In future it will be an important part in our everyday life. Due to this situation it is very important to understand all methods and algorithms, which will come with this technology. To understand the convergence behavior of the ADAM-Optimizer, this paper shows an improvement of the convergence proof of \cite{Kingma.2015}. Unfortunately we have at least one conjecture which is still in question. Hopefully this will be proved in future works, so that we can use the ADAM-Optimizer without any concerns. Probably the whole proof can show us some opportunities in order to improve the algorithm's speed and efficiency, so that the learning time will decrease. Especially in the time of big data this could be a decisive advantage.
\section*{Acknowledgments}
This work was partially supported by Baumann GmbH and MediaMarktSaturn Retail Group GmbH.
\nocite{Kingma.2015}
\nocite{Kruse.2011}
\nocite{Rumelhart.1987}
\nocite{Bock.2017}
\nocite{Goppold.2017}
\bibliographystyle{IEEEtran} 
\bibliography{Literatur}{}

\begin{thebibliography}{1}
\providecommand{\url}[1]{#1}
\csname url@samestyle\endcsname
\providecommand{\newblock}{\relax}
\providecommand{\bibinfo}[2]{#2}
\providecommand{\BIBentrySTDinterwordspacing}{\spaceskip=0pt\relax}
\providecommand{\BIBentryALTinterwordstretchfactor}{4}
\providecommand{\BIBentryALTinterwordspacing}{\spaceskip=\fontdimen2\font plus
\BIBentryALTinterwordstretchfactor\fontdimen3\font minus
  \fontdimen4\font\relax}
\providecommand{\BIBforeignlanguage}[2]{{%
\expandafter\ifx\csname l@#1\endcsname\relax
\typeout{** WARNING: IEEEtran.bst: No hyphenation pattern has been}%
\typeout{** loaded for the language `#1'. Using the pattern for}%
\typeout{** the default language instead.}%
\else
\language=\csname l@#1\endcsname
\fi
#2}}
\providecommand{\BIBdecl}{\relax}
\BIBdecl

\bibitem{Kingma.2015}
D.~P. Kingma and J.~L. Ba, \emph{{Adam: A Method for stochastic
  Optimization}}.\hskip 1em plus 0.5em minus 0.4em\relax San Diego: The
  International Conference on Learning Representations (ICLR), 2015.

\bibitem{Ruder.2017}
\BIBentryALTinterwordspacing
S.~Ruder, ``{An overview of gradient descent optimization algorithms},'' cite
  arxiv:1609.04747Comment: 12 pages, 6 figures. [Online]. Available:
  \url{http://arxiv.org/abs/1609.04747}
\BIBentrySTDinterwordspacing

\bibitem{Goppold.2017}
J.~Goppold, ``{Identifikation von Serverfehlern mittels Support Vector Machines
  und k{\"u}nstlichen neuronalen Netzen},'' Regensburg, 2017.

\bibitem{Bock.2017}
S.~Bock, ``{Rotationsermittlung von Bauteilen basierend auf neuronalen
  Netzen},'' Regensburg, 2017.

\bibitem{Forster.2016}
O.~Forster, \emph{{Analysis}}, 12th~ed., ser. {Grundkurs Mathematik}.\hskip 1em
  plus 0.5em minus 0.4em\relax Braunschweig and Wiesbaden: Vieweg, 2016,
  vol.~1.

\bibitem{Kruse.2011}
R.~Kruse, \emph{{Computational Intelligence: Eine methodische Einf{\"u}hrung in
  K{\"u}nstliche Neuronale Netze, Evolution{\"a}re Algorithmen, Fuzzy-Systeme
  und Bayes-Netze}}, 1st~ed., ser. {Computational Intelligence}.\hskip 1em plus
  0.5em minus 0.4em\relax Wiesbaden: {Vieweg + Teubner}, 2011.

\bibitem{Rumelhart.1987}
\BIBentryALTinterwordspacing
D.~E. Rumelhart and J.~L. McClelland, ``{Learning Internal Representations by
  Error Propagation},'' \emph{{Parallel Distributed Processing:Explorations in
  the Microstructure of Cognition: Foundations}}, pp. 318--362, 1987. [Online].
  Available: \url{http://ieeexplore.ieee.org/stamp/stamp.jsp?arnumber=6302929}
\BIBentrySTDinterwordspacing

\end{thebibliography}

\end{document}